\documentclass[11pt]{article}

\usepackage[margin = 1in]{geometry}
\usepackage{microtype}
\usepackage{graphicx}
\usepackage{booktabs} 
\usepackage{longtable}
\usepackage{wrapfig}
\usepackage{rotating}
\usepackage[normalem]{ulem}
\usepackage{amsmath}
\usepackage{textcomp}
\usepackage{amssymb}
\usepackage{capt-of}
\usepackage{comment}
\usepackage{dsfont}

\usepackage{enumitem}
\usepackage{bm}

\usepackage{amsmath}
\usepackage{amssymb}

\usepackage{mathtools}
\usepackage{amsthm}

\usepackage{subcaption}

\usepackage{hyperref}
\ExplSyntaxOn
\NewDocumentCommand{\foo}{s m}{%
  \IfBooleanTF{#1}{Foo~#2}{#2}%
}
\ExplSyntaxOff

\theoremstyle{plain}
\newtheorem{thm}{\protect\theoremname}
\theoremstyle{plain}
\newtheorem{lem}{\protect\lemmaname}
\theoremstyle{plain}
\newtheorem{cor}{\protect\corollaryname}

\usepackage{babel}
\usepackage{bm}
\usepackage{bbm}

\usepackage[linesnumbered,ruled,vlined]{algorithm2e}

\SetCommentSty{mycommfont}
\usepackage{algorithmic}

\newcommand{\E}{{\mathbb{E}}}

\let\hat\widehat

\let\tilde\widetilde

\allowdisplaybreaks

\usepackage{color}

\newtheorem{remark}{Remark}

\global\long\def\Ccal{\mathcal{C}}%
\global\long\def\var{\mathsf{var}}%
\global\long\def\polylog{\mathsf{polylog}}%

\global\long\def\Frm{{\rm F}}%
\global\long\def\OPT{\mathsf{OPT}}%
\global\long\def\bmu{\bm{\mu}}%
\global\long\def\bmuhat{\hat{\bmu}}%
\global\long\def\bmutilde{\tilde{\bmu}}%
\global\long\def\bAk{\bA^{(k)}}%
\global\long\def\bmuhatk{\bmuhat^{(k)}}%

\global\long\def\Chat{\hat{\Ccal}}%
\global\long\def\Ctilde{\tilde{\Ccal}}%
\global\long\def\Uhat{\hat{U}}%

\global\long\def\wmin{w_{\mathsf{min}}}%
\global\long\def\Delmu{\Delta_{\mu}}%
\global\long\def\DelV{\Delta_{V}}%

\newcommand{\bA}{\bm{A}}

\newcommand{\ba}{\bm{a}}

\newcommand{\bc}{\bm{c}}

\newcommand{\bx}{\bm{x}}

\makeatother

\providecommand{\corollaryname}{Corollary}
\providecommand{\lemmaname}{Lemma}
\providecommand{\theoremname}{Theorem}

\usepackage{authblk}

\title{Clustering Mixtures of Discrete Distributions: A Note on Mitra's Algorithm}

\author[1]{Mohamed~Seif}
 \author[1]{Yanxi Chen}

\affil[1]{\small
Department of Electrical and Computer Engineering, Princeton University}

\date{}

\begin{document}

\maketitle


\begin{abstract}
    In this note, we provide a refined analysis of Mitra's algorithm \cite{mitra2008clustering} for classifying general discrete mixture distribution models. Built upon spectral clustering \cite{mcsherry2001spectral}, this algorithm offers compelling conditions for probability distributions. We enhance this analysis by tailoring the model to bipartite stochastic block models, resulting in more refined conditions. Compared to those derived in \cite{mitra2008clustering}, our improved separation conditions are obtained. 
\end{abstract}

\section{Introduction}
\label{introduction}

Clustering is a critical challenge in network science, pivotal for detecting underlying patterns and structures in unlabeled data. To explore the boundaries of this challenge, stochastic block models (SBMs) have been effectively utilized as a mathematical framework to assess the performance of clustering algorithms. Specifically, an SBM is a statistical model developed to reveal the structural dynamics of networks or graphs, where nodes represent individual entities and edges symbolize the connections between them. In a typical SBM, nodes are categorized into blocks or communities according to their connectivity patterns, with the probability of an edge existing between any two nodes depending on the blocks to which they belong \cite{holland1983stochastic}. For example, in a social network using an SBM, nodes might be organized by attributes such as age, gender, or geographic location, with friendship probabilities determined by their block memberships \cite{abbe2015exact, mossel2015consistency}.

The Bipartite Stochastic Block Model (B-SBM) \cite{florescu2016spectral} extends the conventional SBM to accommodate networks comprising two distinct node types, forming a bipartite graph structure. This adaptation is particularly beneficial in contexts such as recommendation systems, where nodes represent users and products, or in particular social networks, where nodes might denote individuals and the groups or events they participate in. In B-SBMs, the connections between nodes from different sets are governed by an "affinity matrix" that specifies the likelihood of linkage based on group affiliations. This matrix is integral to capturing interaction patterns within the network, allowing for a sophisticated estimation of model parameters from observed connections. 

\subsection{Related Work}

Mitra proposed an elegant clustering method for mixtures of discrete distributions, which Neumann specialized for bipartite graphs \cite{mitra2008clustering}. This two-step algorithm robustly identifies ground-truth clusters within bipartite graphs. It begins by grouping left-side vertices based on neighborhood similarities, then infers right-side clusters from these established groups using degree thresholding. Notably, this method excels at detecting even tiny clusters, characterized by a size of ${O}(n^{\epsilon})$, where $n$ is the number of vertices on the right side, and $\epsilon > 0$. This capability is significant, as traditional algorithms generally only identify clusters of size at least ${\Omega}(\sqrt{n})$.

\section{Setup}

\textbf{Notation:} use $\lesssim,\gtrsim$ to hide constant factors, and $\tilde{O},\tilde{\Omega}$
to hide $\polylog(m,n)$ terms; ``with high probability'' means
with probability at least $1-O(m^{-C})$ for some large constant $C$.

\paragraph{(a) General mixtures of discrete distributions:}

We consider a scenario involving a mixture of \(k\) discrete distributions \(D_{1},\dots,D_{k}\), each characterized by means \(\boldsymbol{\mu}_{1},\dots,\boldsymbol{\mu}_{k} \in [0,1]^{n}\). The separation, denoted by \(\Delta \mu\), signifies the minimum Euclidean distance between any two means, defined as \(\Delta \mu = \min_{r \neq s} \|\boldsymbol{\mu}_{r} - \boldsymbol{\mu}_{s}\|_{2}\). Samples are generated from these distributions by independently drawing \(n\)-dimensional vectors, with each component \(a_{j}\) following a Bernoulli distribution with parameter \(\mu_{r,j}\) for distribution \(D_{r}\). Represented by \(\mathbf{A} \in \{0,1\}^{m \times n}\), a dataset comprises \(m\) such samples, where \(\mathbf{A}(i) \in \{0,1\}^{n}\) denotes the \(i\)-th row. Each distribution \(D_{r}\) contributes \(w_{r}m\) samples to \(\mathbf{A}\), with \(\sum_{r} w_{r} = 1\), and \(\omega_{\text{min}}\) being the minimum of these weights. The expected value \(\mathbb{E}[\mathbf{A}]\) has \(k\) distinct rows \(\boldsymbol{\mu}_{1},\dots,\boldsymbol{\mu}_{k}\), resulting in a rank \(k\) matrix. Moreover, each component \(\mu_{r,j}\) of the means satisfies \(\mu_{r,j} \leq \sigma^{2}\) for \(r \in [k]\) and \(j \in [n]\), where \(\sigma^{2} > 0\). 

\paragraph{(b) Special case --- bipartite stochastic block model (B-SBM):}

The B-SBM entails \(m\) left vertices and \(n\) right vertices, organized into \(k\) clusters each for the left and right sides denoted as \(U_{1},\dots,U_{k}\) and \(V_{1},\dots,V_{k}\) respectively. The edge probability between vertices is characterized by parameters \(p\) and \(q\), resulting in mean vectors \(\boldsymbol{\mu}_{r} \in \{p,q\}^{n}\) for each cluster \(r\). The separation of cluster centers, denoted as \(\Delta \mu\), satisfies \(\Delta \mu^{2} = (p-q)^{2}\Delta V\), where \(\Delta V = \min_{r \neq s} |V_{r} \bigtriangleup V_{s}|\), indicating the minimum number of differing vertices between any two right clusters. The parameters \(p\) and \(q\) are constrained to \(0 \leq p, q \leq 0.5\), and the variance \(\sigma^{2}\) is defined as \(2\max\{p(1-p),q(1-q)\}\), where \(p(1-p)\) and \(q(1-q)\) represent the variances of Bernoulli random variables. It is noteworthy that these conditions, including the constraint on \(0.5\), are crucial for applying results from Mitra's algorithm, as emphasized in \cite{neumann2018bipartite}. Adjustments to the constant \(0.5\) may be made if the definition of \(\sigma^{2}\) is modified accordingly.

%
Our goal is to cluster the $m$ rows of $\bA$ accurately with high probability, as defined next. We next present and analyze Mitra's algorithm \cite{mitra2008clustering} for classifying general discrete distributions in the following section.


%

\section{Mitra's algorithm}

The algorithm takes as input a dataset \(\mathbf{A} \in \{0,1\}^{m \times n}\). It begins by splitting \(\mathbf{A}\) into two equal-sized subsets, \(\mathbf{A}_1\) and \(\mathbf{A}_2\). The algorithm then proceeds in two steps. First, it runs the Centers algorithm with \(\mathbf{A}_1\) to obtain estimated centers, followed by running the Assignment algorithm for \(\mathbf{A}_2\). Next, it repeats the process by running the Centers algorithm with \(\mathbf{A}_2\) to obtain new estimated centers and then running the Assignment algorithm for \(\mathbf{A}_1\). The assigned labels of \(\mathbf{A}_2\) and \(\mathbf{A}_1\) are matched by comparing the two sets of estimated centers, which is necessary due to the labeling function returned by the Assignment algorithm being accurate only up to a permutation. Finally, the algorithm outputs a labeling function \(\hat{C}:[m] \rightarrow [k]\). We next proceed to provide a detailed description of the two procedures.

\underline{Finding Centers:} The algorithm begins with input data \(\mathbf{A} \in \{0,1\}^{m \times n}\). It then computes the \(k\)-SVD (\(k\)-Singular Value Decomposition) of \(\mathbf{A}\), denoted as \(\mathbf{A}_k\). Next, \(k\)-means clustering is applied to the rows of \(\mathbf{A}_k\), resulting in clusters \(\hat{U}_1, \dots, \hat{U}_k\). Finally, the algorithm outputs estimated means \(\hat{\boldsymbol{\mu}}_1, \dots, \hat{\boldsymbol{\mu}}_k\), which are computed using the rows of the original data \(\mathbf{A}\). Specifically, each estimated mean \(\hat{\boldsymbol{\mu}}_r\) is obtained by averaging the rows of \(\mathbf{A}\) belonging to cluster \(\hat{U}_r\), calculated as \(\hat{\boldsymbol{\mu}}_r = \frac{1}{|\hat{U}_r|} \sum_{i \in \hat{U}_r} \mathbf{A}(i)\), where \(1 \leq r \leq k\). We summarize this procedure in Algorithm \ref{alg:centers}.

\begin{algorithm}
  \caption{Finding the Centers}
  \begin{algorithmic}[1]
     \STATE {\bfseries Input:} $\bA\in\{0,1\}^{m\times n}$
       \STATE Let $\bAk$ be the $k$-SVD of $\bA$
  \STATE Apply $k$-means to the rows of $\bAk$, which gives clusters $\Uhat_{1},\dots,\Uhat_{k}$
  \STATE  {\bfseries Output:}  $\bmuhat_{1},\dots,\bmuhat_{k}$, which are estimated with
rows of the original data $\bA$,
\[
\bmuhat_{r}=\frac{1}{|\Uhat_{r}|}\sum_{i\in\Uhat_{r}}\bA(i),\quad1\le r\le k.
\]
  \end{algorithmic}
  \label{alg:centers}
\end{algorithm}

\underline{Assignment:} The algorithm, essentially equivalent to the "Project" method in Mitra's algorithm, operates based on nearest neighbor principles. It takes as input a dataset \(\mathbf{A} \in \{0,1\}^{m \times n}\) along with the estimated centers \(\hat{\boldsymbol{\mu}}_1, \dots, \hat{\boldsymbol{\mu}}_k\). The output is a labeling function \(\hat{C}:[m] \rightarrow [k]\), defined by assigning each data point \(i\) to the cluster whose estimated center is closest in Euclidean distance. Formally, for each \(i\), \(\hat{C}(i)\) is determined as \(\hat{C}(i) = \arg\min_{r \in [k]} \|\mathbf{A}(i) - \hat{\boldsymbol{\mu}}_r\|_2\). We summarize this procedure in Algorithm \ref{alg:assignment}.


\begin{algorithm}
  \caption{Assignment}
  \begin{algorithmic}[1]
     \STATE {\bfseries Input:} $\bA\in\{0,1\}^{m\times n}$, estimated centers $\bmuhat_{1},\dots,\bmuhat_{k}$
   \STATE {\bfseries Output:}  a labeling function $\Chat:[m]\rightarrow[k]$, defined by
\[
\Chat(i)=\arg\min_{r\in[k]}\|\bA(i)-\bmuhat_{r}\|_{2}.
\]
  \end{algorithmic}
  \label{alg:assignment}
\end{algorithm}


\begin{algorithm}
  \caption{The overall algorithm: Clustering}
  \begin{algorithmic}[1]
 \STATE {\bfseries Input:} $\bA\in\{0,1\}^{m\times n}$
\STATE Split $\bA$ into two subsets $\bA_{1},\bA_{2}$ of the same size
\STATE Run Centers with $\bA_{1}$ to obtain estimated centers, then run
Assignment for $\bA_{2}$ 
\STATE Run Centers with $\bA_{2}$ to obtain estimated centers, then run
Assignment for $\bA_{1}$ 
\STATE Match the assigned labels of $\bA_{2}$ and $\bA_{1}$ by matching
the two sets of estimated centers\footnote{This is necessary (though easy to achieve), since the labeling function
returned by Assignment is only accurate up to a permutation.} 
 \STATE {\bfseries Output:} a labeling function $\Chat:[m]\rightarrow[k]$
  \end{algorithmic}
  \label{alg:overall_clustering}
\end{algorithm}

The overall steps of Mitra's clustering algorithm are summarized in Algorithm \ref{alg:overall_clustering}.

\section{Main results }

\subsection{The general case}
\begin{thm}
[General case]\label{thm:general} If the data $\bA\in\{0,1\}^{m\times n}$
and parameters satisfy
\[
\|\bA-\E[\bA]\|^{2}\le\frac{0.01\wmin m\Delmu^{2}}{50k},\quad m\sigma^{2}\gtrsim1,\quad\text{and}\quad\Delmu\ge\tilde{\Omega}\bigg(\sqrt{\frac{\sigma^{2}}{\wmin}}\bigg),
\]
then the overall algorithm achieves exact clustering of $\bA$ with
high probability.
\end{thm}

\begin{remark}
 The requirement \(m\sigma^{2}\gtrsim1\) emerges from an upper bound on column sums of \(\mathbf{A}\). Remarkably, \cite{neumann2018bipartite} overlooked this requirement. It remains uncertain whether Chernoff bounds tailored to sums of Bernoulli random variables could obviate this condition. However, in the context of B-SBM, this condition is necessary to ensure that there is no isolated
right vertices
\end{remark}

Our proof is based on the following performance guarantees for Centers
and Assignment.
\begin{lem}
[Centers]\label{lem:centers} If the data $\bA\in\{0,1\}^{m\times n}$
and parameters satisfy
\[
\|\bA-\E[\bA]\|^{2}\le\frac{0.01\wmin m\Delmu^{2}}{50k}\quad\text{and}\quad m\sigma^{2}\gtrsim1,
\]
then with high probability, the estimated centers $\bmuhat_{1},\dots,\bmuhat_{k}$
returned by Centers (with input $\bA$) obey (up to a permutation)
\begin{equation}
\|\bmuhat_{r}-\bmu_{r}\|_{2}\le7\sqrt{\frac{k}{\wmin m}}\|\bA-\E[\bA]\|\le0.1\Delmu\quad\text{and}\quad\|\bmuhat_{r}\|_{\infty}\le\tilde{O}\bigg(\frac{\sigma^{2}}{\wmin}\bigg),\quad1\le r\le k.\label{eq:condition_bmuhat}
\end{equation}
\end{lem}
\begin{lem}
[Assignment]\label{lem:assignment} If $\Delmu\ge\tilde{\Omega}(\sqrt{\sigma^{2}/\wmin})$,
and the estimated centers $\{\bmuhat_{r}\}$ satisfy (\ref{eq:condition_bmuhat}),
then for any $r\in[k]$, given a fresh sample $\ba\sim D_{r}$, one
has $\arg\min_{s\in[k]}\|\ba-\bmuhat_{s}\|_{2}=r$ with high probability.
\end{lem}
\begin{remark}
Theorem \ref{thm:general} follows from these two lemmas. The only
thing to note is that, given $\|\bmuhat_{r}-\bmu_{r}\|_{2}\le0.1\Delmu$
for both sets of estimated centers $\{\bmuhat_{r}\}$ in the overall
clustering algorithm, it is easy to find the one-to-one correspondence
between these two sets, and thus combining the clustering results
for $\bA_{1}$ and $\bA_{2}$ leads to exact clustering of the full
data $\bA$.

\end{remark}



\subsection{Implications for B-SBM }
\begin{cor}
[Special case: B-SBM]\label{cor:B-SBM} Consider the data $\bA$
generated by the B-SBM model. If
\[
\frac{(p-q)^{2}}{\sigma^{2}}\ge\tilde{\Omega}\bigg(\frac{k(m+n)}{\wmin m\DelV}\bigg)\quad\text{and}\quad\sigma^{2}\ge\frac{\log^{6}n}{n}+\frac{1}{m},
\]
then the overall algorithm achieves exact clustering with high probability.
\end{cor}
\begin{proof}
We simply need to specify each condition in Theorem \ref{thm:general}
for the B-SBM case. By Talagrand's inequality\footnote{Any bound on $\|\bA-\E[\bA]\|$ can be plugged in here.},
if $\sigma^{2}\ge\log^{6}(n)/n$, then with high probability, $\|\bA-\E[\bA]\|^{2}\lesssim\sigma^{2}(m+n)$;
moreover, recall that $\Delmu^{2}=(p-q)^{2}\DelV$. Then, it is easy
to check that 
\[
\frac{(p-q)^{2}}{\sigma^{2}}\gtrsim\frac{k(m+n)}{\wmin m\DelV}\quad\Rightarrow\quad\|\bA-\E[\bA]\|^{2}\le\frac{0.01\wmin m\Delmu^{2}}{50k}.
\]
In addition, for the last condition in Theorem \ref{thm:general},
one has

\[
\Delmu\ge\tilde{\Omega}\bigg(\sqrt{\frac{\sigma^{2}}{\wmin}}\bigg)\quad\iff\quad(p-q)^{2}\DelV\gtrsim\tilde{\Omega}\bigg(\frac{\sigma^{2}}{\wmin}\bigg)\quad\iff\quad\frac{(p-q)^{2}}{\sigma^{2}}\ge\tilde{\Omega}\bigg(\frac{1}{\wmin\DelV}\bigg).
\]
Putting these together completes our proof of the corollary.
\end{proof}

\section{Proofs }

\subsection{Proof of Lemma \ref{lem:centers} (Centers)}

Notation: given clusters $\Uhat_{1},\dots,\Uhat_{k}$, denote the
estimated centers in $k$-means as 
\[
\bmuhatk_{r}=\frac{1}{|\Uhat_{r}|}\sum_{i\in\Uhat_{r}}\bAk(i),\quad1\le r\le k.
\]
Recall that $k$-means is run on the matrix $\bAk$, and hence each
$\bmuhatk_{r}$ is an average of $\{\bAk(i)\}$; in contrast, the
final estimated centers $\bmuhat_{1},\dots,\bmuhat_{k}$ returned
by Centers are computed with the rows of the original matrix $\bA$.
In our analysis, we first work with $\{\bmuhatk_{r}\}$, and then
derive the guarantees for $\{\bmuhat_{r}\}$.

\underline{Step $1$: errors of the estimated centers in $k$-means:}  In this step, we aim to show that
\begin{equation}
\min_{s\in[k]}\|\bmuhatk_{s}-\bmu_{r}\|_{2}^{2}\le\frac{4\|\bAk-\E[\bA]\|_{\Frm}^{2}}{\wmin m}\le\frac{32k\|\bA-\E[\bA]\|^{2}}{\wmin m}\le0.01\Delmu^{2},\quad1\le r\le k.\label{eq:k_means_centers_error}
\end{equation}
If this is true, then $\min_{s\in[k]}\|\bmuhatk_{s}-\bmu_{r}\|_{2}\le0.1\Delmu$,
which implies that there is an one-to-one correspondence between $\{\bmuhatk_{r}\}$
and $\{\bmu_{r}\}$; in other words, the function $f$ defined by
$f(r)=\arg\min_{s}\|\bmuhatk_{s}-\bmu_{r}\|_{2}$ is a permutation
on $[k]$. Without loss of generality, we assume that $f(r)=r$ is
the identity function in our analysis.

Let us first prove the first inequality in (\ref{eq:k_means_centers_error}).
Let $\Ccal:[m]\rightarrow[k]$ be the ground-truth labeling function,
and $\Chat$ be the labeling function returned by $k$-means on $\bAk$.
Denote $\OPT$ as the optimal value of $k$-means. Then $\OPT$ can
be upper bounded by
\begin{align*}
\OPT & =\sum_{i\in[m]}\|\bAk(i)-\bmuhatk_{\Chat(i)}\|_{2}^{2}=\min_{\bmutilde,\Ctilde}\sum_{i\in[m]}\|\bAk(i)-\bmutilde_{\Ctilde(i)}\|_{2}^{2}\\
 & \le\sum_{i\in[m]}\|\bAk(i)-\bmu_{\Ccal(i)}\|_{2}^{2}=\|\bAk-\E[\bA]\|_{\Frm}^{2}.
\end{align*}
Fixing $r\in[k]$ and denoting the targeted error as
\[
\xi :=  \min_{s\in[k]}\|\bmuhatk_{s}-\bmu_{r}\|_{2},
\]
we can lower bound $\OPT$ by
\begin{align*}
\OPT & \ge\sum_{i\in U_{r}}\|\bAk(i)-\bmuhatk_{\Chat(i)}\|_{2}^{2}=\sum_{i\in U_{r}}\|\bAk(i)-\bmu_{r}+\bmu_{r}-\bmuhatk_{\Chat(i)}\|_{2}^{2}\\
 & \overset{{\rm (i)}}{\ge}\sum_{i\in U_{r}}\Big(-\|\bAk(i)-\bmu_{r}\|_{2}^{2}+\frac{1}{2}\|\bmu_{r}-\bmuhatk_{\Chat(i)}\|_{2}^{2}\Big)\overset{{\rm (ii)}}{\ge}-\|\bAk-\E[\bA]\|_{\Frm}^{2}+\frac{\wmin m\xi^{2}}{2},
\end{align*}
where (i) follows from the Cauchy-Schwarz inequality, and (ii) is
due to the assumption that $|U_{r}|\ge\wmin m$, as well as the definition
of $\xi$. Putting both upper and lower bounds for $\OPT$ together
finishes our proof for the first inequality in (\ref{eq:k_means_centers_error}).

The second inequality in (\ref{eq:k_means_centers_error}) is easy
to prove: since $\E[\bA]$ has rank $k$, and $\bAk$ is the best
rank-$k$ approximation of $\bA$, one has 
\[
\|\bAk-\E[\bA]\|\le\|\bAk-\bA\|+\|\bA-\E[\bA]\|\le2\|\bA-\E[\bA]\|;
\]
in addition, $\bAk-\E[\bA]$ has rank at most $2k$, which implies
\[
\|\bAk-\E[\bA]\|_{\Frm}^{2}\le2k\|\bAk-\E[\bA]\|^{2}\le8k\|\bA-\E[\bA]\|^{2}.
\]
Finally, the last inequality in (\ref{eq:k_means_centers_error})
follows from the assumption of the lemma. 

\underline{Step 2: clustering errors in $k$-means:} Based on the results from the previous step, we aim to further show
that $k$-means returns an accurate clustering $\Uhat_{1},\dots,\Uhat_{k}$,
in the sense that
\[
|\Uhat_{r}\cap U_{r}|\ge0.9|U_{r}|\ge0.9\wmin m,\quad1\le r\le k.
\]
To prove this, we start with a lower bound for $\OPT$:
\begin{align*}
\OPT & \ge\sum_{i\in U_{r}\backslash\Uhat_{r}}\|\bAk(i)-\bmuhatk_{\Chat(i)}\|_{2}^{2}=\sum_{i\in U_{r}\backslash\Uhat_{r}}\|\bAk(i)-\bmu_{r}+\bmu_{r}-\bmuhatk_{\Chat(i)}\|_{2}^{2}\\
 & \overset{{\rm (i)}}{\ge}\frac{1}{2}\sum_{i\in U_{r}\backslash\Uhat_{r}}\|\bmu_{r}-\bmuhatk_{\Chat(i)}\|_{2}^{2}-\sum_{i\in U_{r}\backslash\Uhat_{r}}\|\bAk(i)-\bmu_{r}\|_{2}^{2}\\
 & \overset{{\rm (ii)}}{\ge}\frac{1}{2}|U_{r}\backslash\Uhat_{r}|(\frac{9}{10}\Delmu)^{2}-\|\bAk-\E[\bA]\|_{\Frm}^{2},
\end{align*}
where (i) follows from the Cauchy-Schwarz inequality, and (ii) is
due to $\Chat(i)\neq r$ for all $i\in U_{r}\backslash\Uhat_{r}$.
Moreover, recall from the previous step that $\OPT\le\|\bAk-\E[\bA]\|_{\Frm}^{2}$.
Putting both upper and lower bounds together, we arrive at
\[
|U_{r}\backslash\Uhat_{r}|\le\frac{5\|\bAk-\E[\bA]\|_{\Frm}^{2}}{\Delmu^{2}}\le\frac{40k\|\bA-\E[\bA]\|^{2}}{\Delmu^{2}}\le0.1\wmin m\le0.1|U_{r}|,
\]
which implies the desired lower bound on $|\Uhat_{r}\cap U_{r}|$.

\subsubsection{Step 3: the final estimated centers }

Recall the final estimated centers 
\[
\bmuhat_{r}=\frac{1}{|\Uhat_{r}|}\sum_{i\in\Uhat_{r}}\bA(i),\quad1\le r\le k.
\]
We first control the estimation error by 
\[
\|\bmuhat_{r}-\bmu_{r}\|_{2}\le\|\bmuhat_{r}-\bmuhatk_{r}\|_{2}+\|\bmuhatk_{r}-\bmu_{r}\|_{2},
\]
where the second term has been upper bounded by $\sqrt{32k/(\wmin m)}\|\bA-\E[\bA]\|$
in Step 1. For the first term, one has
\begin{align*}
\|\bmuhat_{r}-\bmuhatk_{r}\|_{2} & =\|\frac{1}{|\Uhat_{r}|}\sum_{i\in\Uhat_{r}}(\bA(i)-\bAk(i))\|_{2}\le\frac{1}{\sqrt{|\Uhat_{r}|}}\|\bA-\bAk\|\\
 & \le\sqrt{\frac{1}{0.9\wmin m}}\|\bA-\bAk\|\le\sqrt{\frac{1}{0.9\wmin m}}\|\bA-\E[\bA]\|,
\end{align*}
where the last inequality is because $\bAk$ is the best rank-$k$
approximation of $\bA$. Putting these together, we arrive at $\|\bmuhat_{r}-\bmu_{r}\|_{2}\le7\sqrt{k/(\wmin m)}\|\bA-\E[\bA]\|$.

It remains to analyze $\|\bmuhat_{r}\|_{\infty}$. Recall that $\bA\in\{0,1\}^{m\times n}$,
and the expectation and variance of each entry are both upper bounded
by $\sigma^{2}$. Using Bernstein's inequality (or Chernoff bound)
and a union bound over $n$ columns of $\bA$, we have the following:
with high probability, the sum of each column of $\bA$ is bounded
by 
\[
0\le\sum_{i=1}^{m}A_{i,j}\lesssim m\sigma^{2}+\tilde{O}(\sqrt{m\sigma^{2}}+1)\lesssim\tilde{O}(m\sigma^{2}),\quad1\le j\le n,
\]
where the last inequality is due to the assumption that $m\sigma^{2}\gtrsim1$.
Combining this with $|\Uhat_{r}|\gtrsim\wmin m$, we arrive at $\|\bmuhat_{r}\|_{\infty}\lesssim\tilde{O}(\sigma^{2}/\wmin)$,
which finishes our proof of the lemma.

\subsection{Proof of Lemma \ref{lem:assignment} (Assignment)}

Consider a fresh sample $\ba\sim D_{r}$ that is independent of $\bmuhat_{1},\dots,\bmuhat_{k}$.
We want to show that, with high probability,
\[
\|\ba-\bmuhat_{r}\|_{2}^{2}<\|\ba-\bmuhat_{s}\|_{2}^{2},\quad\forall s\neq r.
\]
For any fixed $s\neq r$, let us start with a decomposition of the
right-hand side:
\begin{align*}
\|\ba-\bmuhat_{s}\|_{2}^{2} & =\|\ba-\bmuhat_{r}+\bmuhat_{r}-\bmuhat_{s}\|_{2}^{2}=\|\ba-\bmuhat_{r}\|_{2}^{2}+\|\bmuhat_{r}-\bmuhat_{s}\|_{2}^{2}+2\langle\ba-\bmuhat_{r},\bmuhat_{r}-\bmuhat_{s}\rangle.
\end{align*}
Thus the desired condition is equivalent to
\[
\|\bmuhat_{r}-\bmuhat_{s}\|_{2}^{2}+2\langle\ba-\bmuhat_{r},\bmuhat_{r}-\bmuhat_{s}\rangle>0,
\]
or
\[
\|\bmuhat_{r}-\bmuhat_{s}\|_{2}^{2}>-2\langle\ba-\bmuhat_{r},\bmuhat_{r}-\bmuhat_{s}\rangle=-2\langle\ba-\bmu_{r},\bmuhat_{r}-\bmuhat_{s}\rangle-2\langle\bmu_{r}-\bmuhat_{r},\bmuhat_{r}-\bmuhat_{s}\rangle.
\]
Hence, it suffices to prove the following two conditions:
\begin{align}
|\langle\bmu_{r}-\bmuhat_{r},\bmuhat_{r}-\bmuhat_{s}\rangle| & <\frac{1}{4}\|\bmuhat_{r}-\bmuhat_{s}\|_{2}^{2},\label{eq:nn_part1}\\
|\langle\ba-\bmu_{r},\bmuhat_{r}-\bmuhat_{s}\rangle| & <\frac{1}{4}\|\bmuhat_{r}-\bmuhat_{s}\|_{2}^{2}.\label{eq:nn_part2}
\end{align}

First, by assumption, we have $\|\bmu_{r}-\bmuhat_{r}\|_{2}\le0.1\Delmu$,
while $\|\bmuhat_{r}-\bmuhat_{s}\|_{2}\ge\|\bmu_{r}-\bmu_{s}\|-2\times0.1\Delmu\ge0.8\Delmu$,
which immediately leads to (\ref{eq:nn_part1}). Thus, it remains
to prove (\ref{eq:nn_part2}). Let us rewrite 
\[
\langle\underset{\eqqcolon\bx}{\underbrace{\ba-\bmu_{r}}},\underset{\eqqcolon\bc}{\underbrace{\bmuhat_{r}-\bmuhat_{s}}}\rangle=\langle\bx,\bc\rangle=\sum_{i=1}^{n}c_{i}X_{i}.
\]
Recall that $\ba\in\{0,1\}^{n}$ has independent Bernoulli entries,
and $\E[\ba]=\bmu_{r}$. Morever, $\|\bc\|_{\infty}=\|\bmuhat_{r}-\bmuhat_{s}\|_{\infty}\le\tilde{O}(\sigma^{2}/\wmin)$
and $\|\bc\|_{2}\ge0.8\Delmu$ by assumption. By Bernstein's inequality,
we have with high probability, 
\[
\sum_{i=1}^{n}c_{i}X_{i}\le\tilde{O}\bigg(\sqrt{\sum_{i=1}^{n}c_{i}^{2}\var(X_{i})}+\|\bc\|_{\infty}\bigg)\le\tilde{O}\bigg(\sigma\|\bc\|_{2}+\frac{\sigma^{2}}{\wmin}\bigg).
\]
Under the assumption that $\|\bc\|_{2}\ge0.8\Delmu\ge\tilde{\Omega}(\sqrt{\sigma^{2}/\wmin})$,
it is easy to check that $\sum_{i}c_{i}X_{i}<\|\bc\|_{2}^{2}/4$,
namely the condition (\ref{eq:nn_part2}) holds. Finally, taking a
union bound over $s\neq r$ completes our proof of the lemma.

\bibliographystyle{plain}
\bibliography{myreferences}

\end{document}